\documentclass[conference]{IEEEtran}
\IEEEoverridecommandlockouts
\usepackage{geometry}
\usepackage{cite}
\usepackage{amsmath,amssymb,amsfonts}
\usepackage{amsthm}
\usepackage{graphicx}
\usepackage{textcomp}
\usepackage{booktabs}
\usepackage{xcolor}

\usepackage{subfigure}
\usepackage[ruled,lined,linesnumbered]{algorithm2e}
\usepackage{algpseudocode}

\newtheorem{theorem}{Theorem}
\newtheorem{assumption}{Assumption}

\def\BibTeX{{\rm B\kern-.05em{\sc i\kern-.025em b}\kern-.08em
    T\kern-.1667em\lower.7ex\hbox{E}\kern-.125emX}}
\begin{document}
\makeatletter
\newcommand{\linebreakand}{%
  \end{@IEEEauthorhalign}
  \hfill\mbox{}\par
  \mbox{}\hfill\begin{@IEEEauthorhalign}
}
\makeatother

\title{FedUHB: Accelerating Federated Unlearning via Polyak Heavy Ball Method\\
}

\author{\IEEEauthorblockN{Yu Jiang}
\IEEEauthorblockA{\textit{Nanyang Technological University} \\
Singapore \\
yu012@e.ntu.edu.sg} \\
\and
\IEEEauthorblockN{Chee Wei Tan}
\IEEEauthorblockA{\textit{Nanyang Technological University} \\
Singapore\\
cheewei.tan@ntu.edu.sg} \\
\and
\IEEEauthorblockN{Kwok-Yan Lam}
\IEEEauthorblockA{\textit{Nanyang Technological University} \\
Singapore\\
kwokyan.lam@ntu.edu.sg}
}
\maketitle

\begin{abstract}

Federated learning facilitates collaborative machine learning, enabling multiple participants to collectively develop a shared model while preserving the privacy of individual data. The growing importance of the ``right to be forgotten" calls for effective mechanisms to facilitate data removal upon request. In response, federated unlearning (FU) has been developed to efficiently eliminate the influence of specific data from the model. Current FU methods primarily rely on approximate unlearning strategies, which seek to balance data removal efficacy with computational and communication costs, but often fail to completely erase data influence.
To address these limitations, we propose FedUHB, a novel exact unlearning approach that leverages the Polyak heavy ball optimization technique, a first-order method, to achieve rapid retraining. In addition, we introduce a dynamic stopping mechanism to optimize the termination of the unlearning process.
Our extensive experiments show that FedUHB not only enhances unlearning efficiency but also preserves robust model performance after unlearning. Furthermore, the dynamic stopping mechanism effectively reduces the number of unlearning iterations, conserving both computational and communication resources. FedUHB can be proved as an effective and efficient solution for exact data removal in federated learning settings.

\end{abstract}


\section{Introduction}
Federated learning (FL) \cite{ mcmahan2017federated, kairouz2021advances, boyd2004convex, liu2022privacy, li2020federated,zhang2024coded,liu2020accelerating,liu2022decentralized, boyd2011distributed} represents a significant advancement in collaborative machine learning, enabling multiple participants to jointly develop a shared model while maintaining the privacy of each participant's training data without the need to centralize sensitive information. However, with the increasing adoption of FL, there arises a crucial need to comply with data protection regulations such as the General Data Protection Regulation (GDPR) \cite{regulation2018general} and the California Consumer Privacy Act (CCPA) \cite{goldman2020introduction}. These regulations grant individuals specific rights regarding their personal data, including the ``right to be forgotten", which mandates the removal of an individual's data from the model upon request.

To address this requirement, the concept of federated unlearning (FU) has been developed to remove the influence of specific data from a model \cite{liu2021federaser, jiang2024towards, che2023fast, zhang2023fedrecovery, liu2024threats}. Current strategies in FU predominantly utilize approximate unlearning methods, which seek to find a balance between unlearning effectiveness and resource consumption, avoiding the need for complete model retraining. For instance, FedEraser \cite{liu2021federaser} calibrates historical gradients that the server intermittently collects and stores during FL rounds to erase the influence of removed data. FedAF \cite{li2023federated} replaces data targeted for removal with synthetic data that mimics the statistical properties of the unlearned data, which helps the model to “forget” the original data by diluting its influence. The method in \cite{wu2022federated} removes a client’s contribution by subtracting historical updates and restores model performance using knowledge distillation.
However, these approximate unlearning methods often have limitations due to the intrinsic properties of knowledge permeability in FL \cite{liu2023survey}, where the influence of specific data points can be diffused but not completely eliminated from the model. This incomplete data removal can result in residual data influence persisting within the model’s parameters, posing significant privacy risks.

\begin{figure*}[t]
    \centering
   \includegraphics[width=\linewidth]{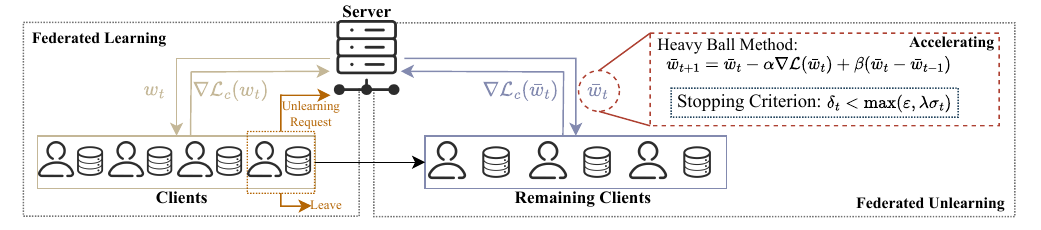}
    \caption{An illustration of FedUHB Scheme. FedUHB is an exact unlearning method that accelerates convergence using Polyak heavy ball optimization and features a dynamic stopping mechanism to optimize the termination of the unlearning process.}
    \label{fig:Workflow}
\end{figure*}

To address these challenges, we adopt an exact unlearning method, which aims to fully eliminate the influence of specified data from the model. The most straightforward approach for exact unlearning is the retrain method, which involves training the model from scratch after removing the unlearned data. However, this method is time-consuming and resource-intensive. Hence, we propose an innovative rapid retraining strategy that leverages the Polyak heavy ball optimization technique \cite{polyak1964some, chen2022communication, ghadimi2015global}, a first-order method, to accelerate model convergence \cite{beck2017first, ryu2022large, nesterov2013gradient, parikh2014proximal,zhu2023federated}.
While the Polyak heavy ball method is seldom used in FL due to its requirement to store historical information from previous rounds, which increases storage and communication overhead, our approach can leverage its momentum feature to accelerate model convergence and enhance stability during the unlearning process. This method incurs only one additional round of storage cost, significantly reducing overall storage and communication expenses compared to methods that rely on extensive historical information for calibration \cite{liu2021federaser,cao2023fedrecover}. By integrating this method into our federated unlearning framework, we enable more precise and faster data removal, offering a practical and robust solution to achieve exact data removal in FL environments. Additionally, we have designed a dynamic stopping mechanism to assist the server in terminating the training process at the appropriate time by tracking the standard deviation of model changes to set precise stopping criteria, ensuring optimal model performance while minimizing unnecessary iterations.

To verify our approach, we conduct comprehensive experiments to evaluate the performance of FedUHB across various settings. Our findings indicate that FedUHB not only accelerates the unlearning process but also maintains strong model performance after unlearning. 

\textbf{Our Contributions.} The main contributions of this work are as follows.
\begin{enumerate}
\item  We introduce FedUHB, an exact unlearning method that incorporates a rapid retraining strategy utilizing the Polyak heavy ball optimization technique, facilitating efficient and effective data removal.
\item We design a dynamic stopping mechanism to optimize the unlearning process by tracking the standard deviation of model changes to set precise stopping criteria.
\item Extensive experiments demonstrate that FedUHB achieves fast convergence and effectively ensures exact data removal.
\end{enumerate}

\section{FedUHB Scheme}
In this section, we provide a detailed description of the FedUHB scheme, from federated learning to federated unlearning. We explain how Polyak heavy ball optimization technique is integrated and how a stopping mechanism is designed during the federated unlearning process. The complete workflow is illustrated in Fig. \ref{fig:Workflow}.
\subsection{Federated Learning}
In the federated learning setting, the system consists of a central server $S$ and a total of $C$ clients. The server's role is to coordinate the training of a global model $w$ utilizing datasets distributed among the clients. Each client $c \in \mathcal{C}$ possesses a dataset $D_c$, collectively forming the complete dataset $D$. The objective is to minimize the loss function, $\min_w \mathcal{L}(D;w) = \min_w \sum_{c=1}^C \mathcal{L}(D_c;w)$, where $\mathcal{L}$ typically represents an empirical loss function, such as cross-entropy loss. For simplicity, we denote $\mathcal{L}_c(w)=\mathcal{L}(D_c;w)$.
Training occurs in iterative rounds. Specifically, in the $t$-th round, the server broadcasts the current global model $w_t$ to the clients. The clients then independently train on their datasets, calculating updates $\nabla \mathcal{L}_c(w_t)$ using stochastic gradient descent. These updates are sent back to the server, which are aggregated by a specific rule $\mathcal{A}$ to refine the global model \cite{mcmahan2017communication, liu2022efficient,liu2023long,liu2024dynamic}. The updated global model $w_{t+1} = w_t - \eta \mathcal{A}(\nabla \mathcal{L}_1(w_t), \nabla \mathcal{L}_2(w_t), \cdots, \nabla \mathcal{L}_C(w_t))$ is then sent back to the clients for further training in the subsequent rounds, with $\eta$ representing the learning rate. This process is repeated until the training meets a specified convergence criterion or achieves the intended training objective, resulting in the final model. 

\subsection{Federated Unlearning}
In the FedUHB approach, the complete workflow for handling unlearning is a meticulously designed process. When the target client $c \in \mathcal{C}_u$ decides to unlearn the data, they initiate an unlearning request, which results in the immediate withdrawal of their data from the training set.

Following the withdrawal, the server reorganizes the remaining clients $c \in \mathcal{C}_r$ for the unlearning process, where $\mathcal{C}_r = \mathcal{C} \setminus \mathcal{C}_u$. At the $t$-th round of the unlearning stage, the server broadcasts the current model state, denoted as $\bar w_t$, to all remaining clients in $\mathcal{C}_r$. Each remaining client then computes a local update $\nabla \mathcal{L}_c(\bar w_t)$ using their own datasets. These updates are calculated based on a loss function that is the same as that used in standard federated learning and are subsequently sent back to the server.
After collecting these local updates, the server aggregates them via a specific rule $\mathcal{A}$: $\nabla \mathcal{L}(\bar w_t) =  \mathcal{A}(\nabla \mathcal{L}_1(\bar w_t), \nabla \mathcal{L}_2(\bar w_t), \cdots, \nabla \mathcal{L}_{|C_r|}(\bar w_t))$, which is similar to the aggregation used in the federated learning.


\textbf{Accelerating unlearning.}
To accelerate the unlearning, the server adopts the Polyak heavy ball method to update models. 
In FL, the Polyak heavy ball method is rarely used due to the need to store historical information from previous rounds, which increases storage and communication overhead. In contrast, FU benefits from the momentum feature of the Polyak heavy ball method, allowing for faster model convergence and improved stability during unlearning, while only incurring an additional round of storage cost.
Specifically, the unlearned model is updated using the Polyak heavy ball method as follows:
\begin{equation}
\bar w_{t+1} = \bar w_t -\alpha \nabla \mathcal{L}(\bar w_t) + \beta (\bar w_t - \bar w_{t-1}),
\end{equation}
where $\alpha$ represents the step size and $\beta \in (0,1)$ denotes the momentum, contributing to the acceleration of convergence by integrating the influence of the previous update’s direction.

\textbf{Setting dynamic stopping criterion.}
In the FedUHB method, a monitoring mechanism is implemented to optimize the timing of training cessation, which ensures that the model achieves optimal performance while reducing unnecessary training iterations. 
Specifically, we establish a dynamic stopping criterion based on the standard deviation. Let $ \delta_t $ represent the norm of the weight change in round $ t $:
\begin{equation}
\delta_t = \|\beta (\bar w_t - \bar w_{t-1})\|_2,
\end{equation}
which is tracked over multiple rounds to compute its standard deviation:
\begin{equation}
\sigma_t = \sqrt{\frac{1}{k} \sum_{j=t-k+1}^t (\delta_j - \overline{\delta}_{t,k})^2},
\end{equation}
where $ \overline{\delta}_{t,k} $ is the average of the weight changes over the last $ k $ rounds:
\begin{equation}
\overline{\delta}_{t,k} = \frac{1}{k} \sum_{j=t-k+1}^t \delta_j.
\end{equation}
The stopping criterion is thus defined based on the standard deviation of the weight changes from the recent $ k $ rounds. Training should cease if the weight change in the latest round \( \delta_t \) is less than a multiple \( \lambda \) of the standard deviation or if it falls below a minimum threshold. Specifically, training stops if \( \delta_t < \max(\varepsilon, \lambda \sigma_t) \), where $\varepsilon$ is the minimum threshold.
In summary, this method optimizes the training process by dynamically monitoring weight changes, effectively balancing performance and efficiency, and ensuring the model stops training at the appropriate time for optimal results. The unlearning algorithm is shown in Algorithm \ref{alg:afu}.

\begin{algorithm}[t]
\SetAlgoNoEnd
\caption{FedUHB Scheme}
\label{alg:afu}
\KwIn{Remaining client $c \in \mathcal{C}_r$ with dataset $D_c$, initial model $\bar w_0$, stopping criterion $\lambda$, step size $\alpha$, momentum $\beta$, minimum threshold $\varepsilon$}
\SetKwFunction{FU}{\textbf{Accelerating Unlearning}}
\SetKwProg{Fn}{}{:}{\KwRet}
\Fn{\FU}{
\textbf{Server executes:} \\
Reinitialize the global model $\bar w_0$; \\
Send global model $\bar w_t$ to all client; \\
\textbf{Remaining client executes:} \\
\For{all remaining client $c_r$}{
Compute and select update to obtain $\nabla \mathcal{L}_c(\bar w_t)$;
}
\textbf{Server executes:} \\
Aggregate updates: $\nabla \mathcal{L}(\bar w_t) \leftarrow \mathcal{A}(\nabla \mathcal{L}_1(\bar w_t), \nabla \mathcal{L}_2(\bar w_t), \cdots, \nabla \mathcal{L}_{|C_r|}(\bar w_t))$;\\
Update model: $\bar w_{t+1} \leftarrow \bar w_t -\alpha \nabla \mathcal{L}(\bar w_t) + \beta (\bar w_t - \bar w_{t-1})$;\\
Compute weight change: $\delta_t = \|\beta (\bar w_t - \bar w_{t-1})\|_2$;\\
Compute standard deviation: $\sigma_t = \sqrt{\frac{1}{k} \sum_{j=t-k+1}^t (\delta_j - \overline{\delta}_{t,k})^2}$;\\
Decide to stop: \\
\If{ \( \delta_t < \max(\varepsilon, \lambda \sigma_t) \)}{
$\bar w =\bar w_t$ }
}
\KwRet{$\bar w$}
\end{algorithm}

\section{Theoretic Analysis}

In this section,  we first outline the assumptions for our theoretical analysis and then analyze the model difference between FedUHB and the retrain method.

\begin{assumption}\label{ass:convex}
The function \( \mathcal{L}(w) \) satisfies the strong convexity condition with constant \( \mu > 0 \), which can be expressed as:
\[
\mathcal{L}(w_2) \geq \mathcal{L}(w_1) + \langle \nabla \mathcal{L}(w_1), w_2 - w_1 \rangle + \frac{\mu}{2} \|w_2 - w_1\|_2^2
\]
for all vectors \( w_1 \) and \( w_2 \).
\end{assumption}

\begin{assumption}\label{ass:smooth}
The gradient of \( \mathcal{L}(w) \) is Lipschitz continuous with a Lipschitz constant \( L \). This implies that for any vectors \( w_1 \) and \( w_2 \), the function satisfies:
\[
\|\nabla \mathcal{L}(w_1) - \nabla \mathcal{L}(w_2)\|_2 \leq L \|w_1 - w_2\|_2.
\]
\end{assumption}

\begin{assumption}\label{ass:difference}
    The stochastic gradient $\nabla \mathcal{L}_c(w)$ has upper bound $G$ and is an unbiased estimator of loss function $\mathcal{L}(w)$:
    \begin{equation}
        \Vert \nabla \mathcal{L}_c(w) \Vert_2 \leq G, \nabla \mathcal{L}(w) = \mathbb{E}\{\nabla \mathcal{L}_c(w)\}.
    \end{equation}
\end{assumption}

\begin{theorem}
The upper bound of the model difference between FedUHB and retrain after $t$ rounds can be expressed as:
\begin{equation}
\begin{aligned}
        &\Vert \bar{w}_{t} - \hat{w}_{t} \Vert_2 \\
    \leq& (\sqrt{1- \alpha \mu})^t \|\bar{w}_{0} - \hat{w}_{0}\|_2 + \frac{\alpha G[1-(\sqrt{1- \alpha \mu})^t]}{(1 - \beta)(1-\sqrt{1- \alpha \mu})},
\end{aligned}
\end{equation}
where $\bar w_{t}$ and $\hat{w}_{t}$ are the global models unlearned by FedUHB and the retrain method after $t$ rounds, respectively.
\end{theorem}

\begin{proof}
Recall that the global model unlearned by FedUHB is updated according to the following rule:
\(\bar{w}_{t+1} = \bar{w}_t - \alpha \nabla \mathcal{L}(\bar{w}_t) + \beta (\bar{w}_t - \bar{w}_{t-1})\).
Let \(\hat{w}\) denote the global model that has been unlearned using the retrain method. The update for the model can then be expressed as \(\hat{w}_{t+1} = \hat{w}_t - \alpha \nabla \mathcal{L}(\hat{w}_t)\), where \(\alpha\) is the learning rate.

Next, we analyze the difference bound between FedUHB and the retrain method:
\begin{equation}
\begin{aligned}
    &\Vert \bar{w}_{t+1} - \hat{w}_{t+1} \Vert_2 \\
    =& \Vert (\bar{w}_t - \alpha \nabla \mathcal{L}(\bar{w}_t) + \beta (\bar{w}_t - \bar{w}_{t-1})) - (\hat{w}_t - \alpha \nabla \mathcal{L}(\hat{w}_t)) \Vert_2   \\
    \leq & \Vert (\bar{w}_{t} - \hat{w}_{t}) - \alpha (\nabla \mathcal{L}(\bar{w}_t) - \nabla \mathcal{L}(\hat{w}_t)) \Vert_2 + \beta \Vert \bar{w}_t - \bar{w}_{t-1}\Vert_2.
\end{aligned}
\end{equation}
For simplicity, we denote $\mathcal{X} = \Vert (\bar{w}_{t} - \hat{w}_{t}) - \alpha (\nabla \mathcal{L}(\bar{w}_t) - \nabla \mathcal{L}(\hat{w}_t)) \Vert_2$. Then, we calculate $\mathcal{X}^2$:
$\mathcal{X}^2 = \|\bar{w}_{t} - \hat{w}_{t}\|_2^2 + \alpha^2 \|\nabla \mathcal{L}(\bar{w}_t) - \nabla \mathcal{L}(\hat{w}_t)\|_2^2  - 2\alpha \langle \bar{w}_{t} - \hat{w}_{t}, \nabla \mathcal{L}(\bar{w}_t) - \nabla \mathcal{L}(\hat{w}_t) \rangle.$
According to Assumption \ref{ass:convex} and Assumption \ref{ass:smooth}, we can obtain:
\begin{equation}
        \mathcal{X}^2 \leq (1- \alpha \mu)\|\bar{w}_{t} - \hat{w}_{t}\|_2^2 + (\alpha^2 - \frac{\alpha}{L}) \|\nabla \mathcal{L}(\bar{w}_t) - \nabla \mathcal{L}(\hat{w}_t)\|_2^2.
\end{equation}
When $0 < \alpha L \leq 1$, we can obtain the following inequality: $\mathcal{X} \leq \sqrt{1- \alpha \mu} \|\bar{w}_{t} - \hat{w}_{t}\|_2$.
Then, we need to bound the \(\|\bar{w}_t - \bar{w}_{t-1}\|_2\) according to the Assumption \ref{ass:difference}:
\begin{equation}
\begin{aligned}
    \|\bar{w}_t - \bar{w}_{t-1}\|_2 
    &\leq \alpha \|\nabla \mathcal{L}(\bar{w}_{t-1})\|_2 + \beta \|\bar{w}_{t-1} - \bar{w}_{t-2}\|_2\\
    &\leq \alpha G + \beta \|\bar{w}_{t-1} - \bar{w}_{t-2}\|_2.
\end{aligned}
\end{equation}
To expand iteratively, we can derive that $\|\bar{w}_t - \bar{w}_{t-1}\|_2 \leq \alpha G \frac{1 - \beta^{t+1}}{1 - \beta}$. Since $0<\beta<1$, we can get $\|\bar{w}_t - \bar{w}_{t-1}\|_2 \leq \frac{\alpha G}{1 - \beta}$. Combined with $\mathcal{X} \leq \sqrt{1- \alpha \mu} \|\bar{w}_{t} - \hat{w}_{t}\|_2$, we can get
\begin{equation} \label{eq:difference}
        \Vert \bar{w}_{t+1} - \hat{w}_{t+1} \Vert_2
        \leq \sqrt{1- \alpha \mu} \|\bar{w}_{t} - \hat{w}_{t}\|_2 + \frac{\alpha G}{1 - \beta}.
\end{equation}
By applying Eq. \ref{eq:difference} recursively, we can have the following bound for any $t \geq 0$:
\begin{equation}
    \Vert \bar{w}_{t} - \hat{w}_{t} \Vert_2
    \leq (\sqrt{1- \alpha \mu})^t \|\bar{w}_{0} - \hat{w}_{0}\|_2 + \frac{\alpha G[1-(\sqrt{1- \alpha \mu})^t]}{(1 - \beta)(1-\sqrt{1- \alpha \mu})},
\end{equation}
where $\bar w_{t}$ and $\hat{w}_{t}$ are the global models unlearned by FedUHB and the retrain method after $t$ rounds, respectively.

\end{proof}
The theorem establishes an upper bound on the model difference between the FedUHB method and the retrain approach after $t$ rounds, demonstrating the unlearning effectiveness of FedUHB.
Besides, a higher \(\beta\) leads to a larger bound on the model difference, while a lower \(\beta\) reduces the difference between FedUHB and the retrain method.
When $0 < \alpha L \leq 1$, the difference stabilizes at 
$\lim_{t \to \infty} \|\bar{w}_{t} - \hat{w}_{t}\|_2 = \frac{\alpha G}{(1 - \beta)(1 - \sqrt{1 - \alpha \mu})}$ as \( t \to \infty \), highlighting the trade-off between convergence speed and unlearning effectiveness.

\section{Experiment}
\subsection{Evaluation Setup}
\subsubsection{Compared methods}

We assess the performance of FedUHB through a comparative analysis with three unlearning strategies: 
(i) Retrain,
(ii) FedRecover \cite{cao2023fedrecover}, which recovers the global model by estimating client updates using historical data and the L-BFGS algorithm (to focus on L-BFGS recovery performance, we did not apply the periodic correction described in FedRecover), and
(iii) FedEraser \cite{liu2021federaser}, which calibrates historical data selected at predefined intervals to enable unlearning.

\subsubsection{Implementation details}


In our simulated FL environment comprising 20 clients, we adopt FedAvg \cite{mcmahan2017communication} as the model aggregation algorithm. The experimental setup includes 5 local training epochs, 40 global training epochs, a learning rate of 0.005, a batch size of 64, a momentum parameter of \(\beta = 0.9\), a stopping criterion of \(\lambda = 0.6\), and 2 target clients requesting unlearning.

\subsubsection{Evaluation metrics}
We start by using standard metrics such as accuracy and loss on the test dataset to evaluate unlearning efficiency. 
Then, we evaluate unlearning impact through two types of attacks: Membership Inference Attack (MIA) \cite{shokri2017membership} and Backdoor Attack (BA) \cite{liu2020reflection}. Membership Inference Success Rate (MISR) assesses unlearning by using a shadow model trained on logits from a compromised global model to distinguish target clients' training data from test data. After unlearning, the shadow model should classify target clients' data as 0. An MISR close to 0.5 suggests successful unlearning. Attack Success Rate (ASR) measures the effectiveness of backdoor unlearning by calculating the proportion of predictions matching the target label when backdoor triggers are used. A lower ASR implies better unlearning effectiveness.

\subsubsection{Datasets and models}
We evaluate the methods using the MNIST \cite{deng2012mnist} and CIFAR-10 \cite{krizhevsky2009learning} datasets. For MNIST, we employ a CNN architecture with 2 convolutional layers, ReLU activation, max pooling, and 2 fully connected layers. For CIFAR-10, we adjust the CNN to 3 input channels and expand the fully connected layers to 1600 units.


\begin{table}[t]
\centering
\caption{Running time (sec) of each unlearning round}
\label{tab:running time}
\resizebox{0.9\columnwidth}{!}{%
\begin{tabular}{@{}ccccc@{}}
\toprule
              & Retrain & FedEraser & FedRecover & \textbf{FedUHB} \\ \midrule
MNIST         & 1.26     & 1.47     & 2.00       & \textbf{0.77}  \\
CIFAR-10      & 16.89    & 18.58    & 20.01      & \textbf{12.51} \\ \bottomrule
\end{tabular}%
}
\end{table}

\begin{figure}[tbp]
	\centering
	\subfigure[Test loss on MNIST]{
		\begin{minipage}{0.42\columnwidth} 
            \includegraphics[width=\textwidth]{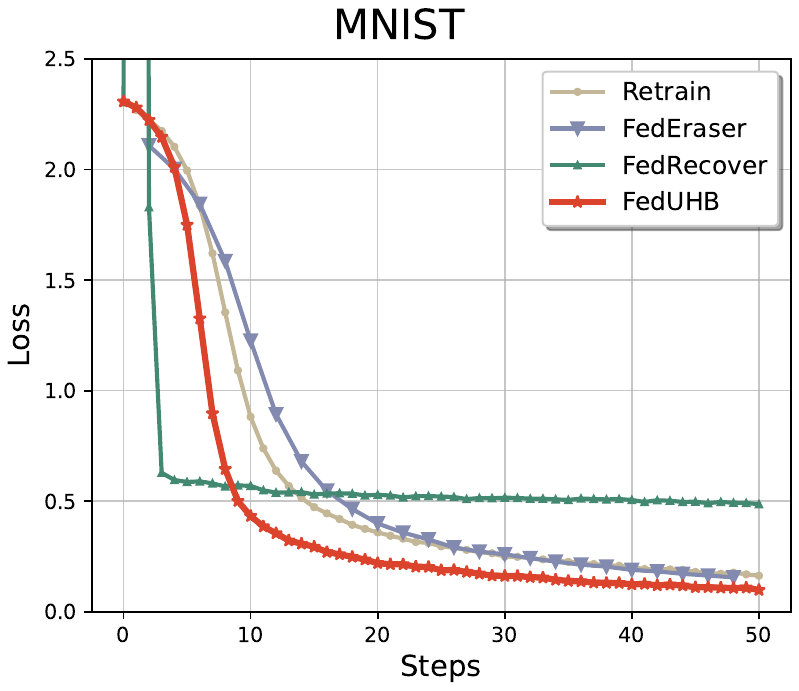} \\
            \label{fig:loss on mnist}
		\end{minipage}
	}
	\subfigure[Test loss on CIFAR-10]{
		\begin{minipage}{0.42\columnwidth}
			\includegraphics[width=\textwidth]{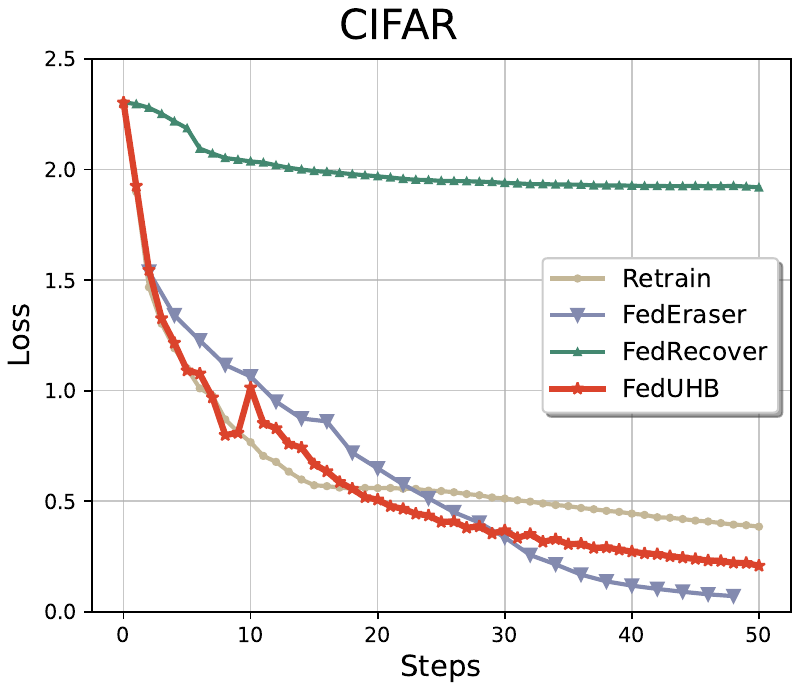}\\
			\label{fig:loss on cifar10}
		\end{minipage}
	}
	\caption{Loss curve of the unlearning on MNIST and CIFAR-10}
    \label{fig:converge}
\end{figure}
\subsection{Evaluation Results and Analysis}

\subsubsection{Running time}
We compare the time required per unlearning round for four methods across two datasets, as shown in Table \ref{tab:running time}. 
On the MNIST dataset, FedUHB demonstrates the shortest runtime, outperforming the retrain method by 0.5 seconds. FedEraser, which calibrates historical information for approximate unlearning, takes 0.7 seconds longer per round than FedUHB. Similarly, FedRecover, which uses the L-BFGS algorithm for unlearning, is slower in practice than our method using Polyak heavy ball optimization, requiring more than twice the time per round compared to FedUHB. On the CIFAR-10 dataset, FedUHB’s runtime is significantly faster than the other three methods, being 4 seconds quicker than the retrain method and 7.5 seconds faster than FedRecover.



\subsubsection{Convergence rate of unlearning}
We compare the convergence performance of four methods on two datasets, as shown in Figure \ref{fig:converge}. 
On the MNIST dataset, both FedEraser and the retrain method exhibit slow convergence, while FedRecover converges more quickly but falls short of optimal performance. In contrast, FedUHB achieves satisfactory convergence within approximately 20 rounds. On the CIFAR-10 dataset, FedRecover’s convergence performance is notably weaker compared to its performance on MNIST. However, FedUHB continues to demonstrate the ideal performance.


\begin{figure}[tbp]
	\centering
	\subfigure[Test accuracy on MNIST]{
		\begin{minipage}{0.42\columnwidth} 
            \includegraphics[width=\textwidth]{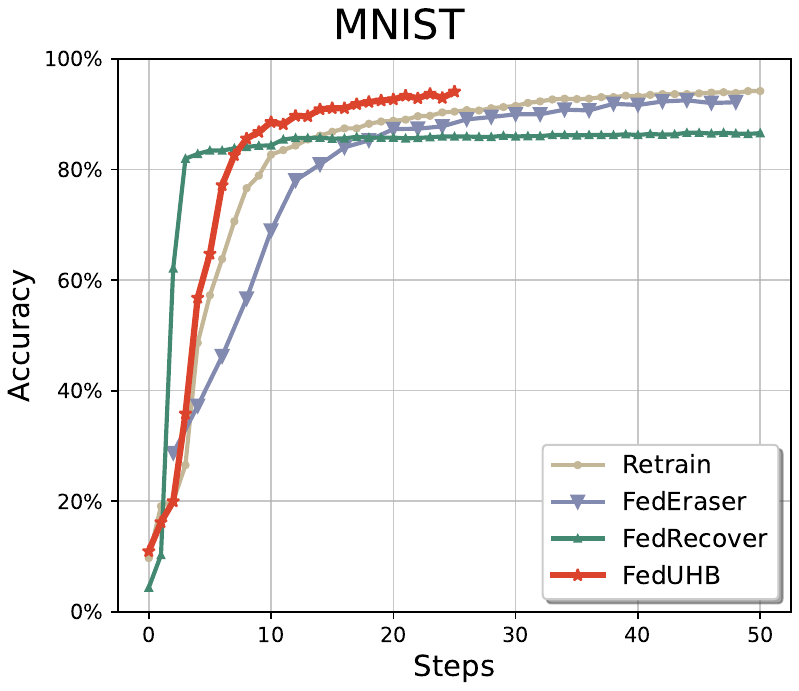} \\
            \label{fig:accuracy on mnist}
		\end{minipage}
	}
	\subfigure[Test accuracy on CIFAR-10]{
		\begin{minipage}{0.42\columnwidth}
			\includegraphics[width=\textwidth]{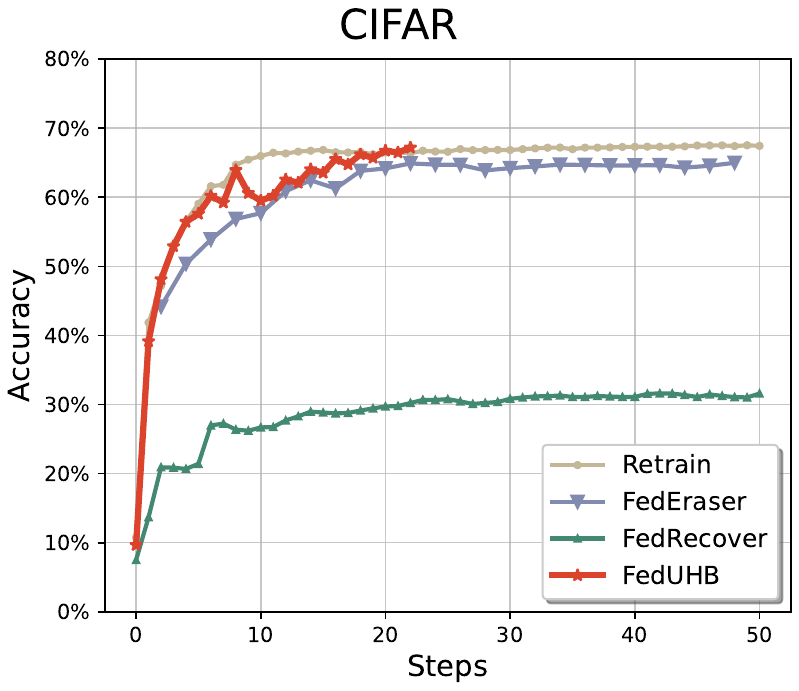}\\
			\label{fig:accuracy on cifar10}
		\end{minipage}
	}
	\caption{Test accuracy of the unlearning on MNIST and CIFAR-10}
    \label{fig:accuracy}
\end{figure}

\subsubsection{Accuracy on the test data}
We evaluate the test accuracy of four methods on two datasets, as shown in Figure \ref{fig:accuracy}. 
On the MNIST dataset, all four methods achieve satisfactory accuracy after unlearning, with FedUHB showing the best results. On the CIFAR-10 dataset, however, there are significant differences in unlearning accuracy. FedRecover performs poorly, achieving only around 30\% accuracy, while the other three methods also experience a decline. Despite experiencing slight fluctuations, FedUHB ultimately surpasses the other methods in accuracy. Thus, in terms of both speed and effectiveness, FedUHB demonstrates outstanding performance.




\begin{figure}[tbp]
	\centering
	\subfigure[MISR \& ASR on MNIST]{
		\begin{minipage}{0.44\columnwidth} 
            \includegraphics[width=\textwidth]{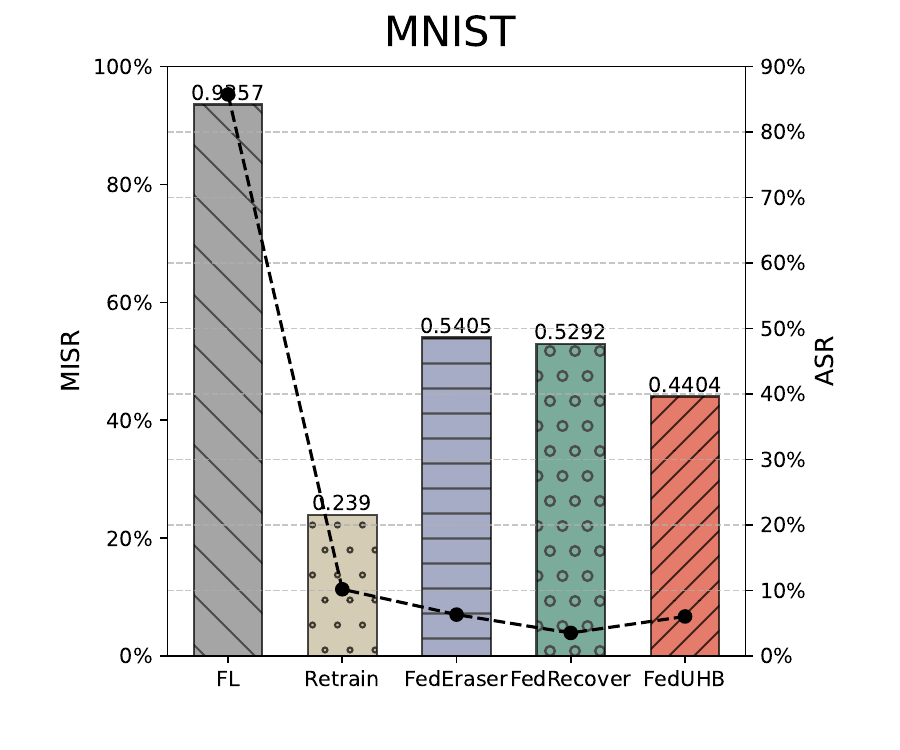} \\
            \label{fig:sr on mnist}
		\end{minipage}
	}
	\subfigure[MISR \& ASR on CIFAR-10]{
		\begin{minipage}{0.44\columnwidth}
			\includegraphics[width=\textwidth]{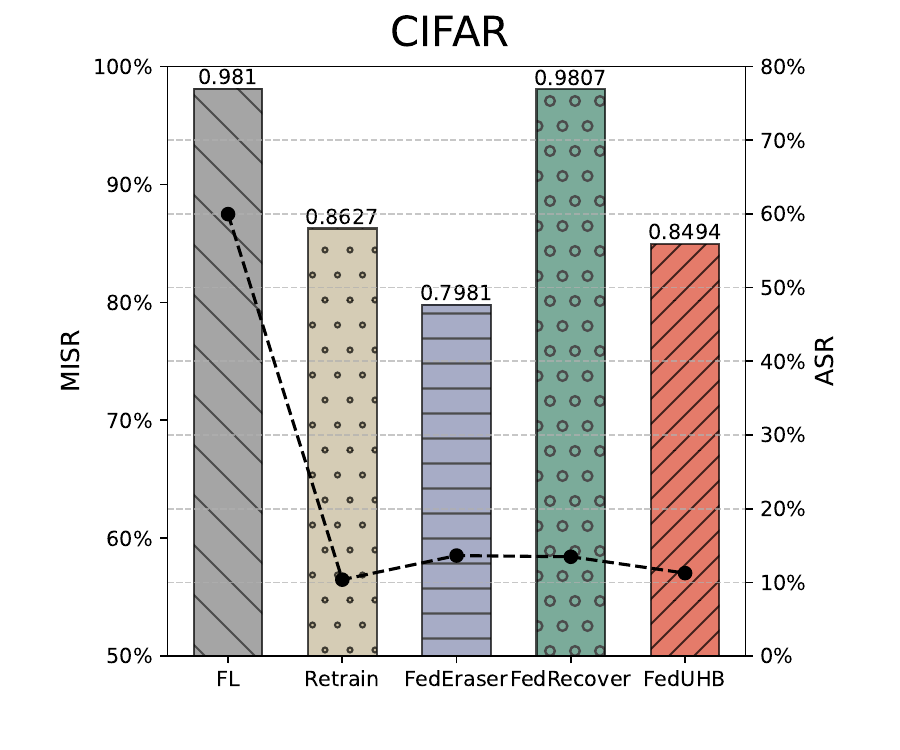}\\
			\label{fig:sr on cifar10}
		\end{minipage}
	}
	\caption{MISR (Bar) and ASR (Line) on MNIST and CIFAR-10}
    \label{fig:attack}
\end{figure}
\subsubsection{Performance after MIA and BA}
We assess the performance after MIA and BA to evaluate the effectiveness of unlearning, shown in Fig. \ref{fig:attack}. 
For MIAs on the MNIST dataset, the MISR of the retrain method is only 23.9\%, indicating a strong capability to effectively remove unlearned clients' data. Meanwhile, the MISRs of FedEraser and FedRecover are close to 50\%, demonstrating that these two methods also effectively mitigate the impact of unlearned data. Notably, FedUHB achieves a lower MISR compared to FedEraser and FedRecover, further highlighting its superior performance. On the CIFAR-10 dataset, while FedUHB's performance decreases slightly, it still surpasses that of the retrain method, underscoring its robustness.
For BAs on the MNIST dataset, all four methods successfully prevent attacks, with FedUHB maintaining an ASR of 3.35\%, indicating that the backdoor attack is largely ineffective and the unlearned data has been successfully removed. On CIFAR-10, the ASRs of FedUHB and the retrain method are lower than those of the other two methods, demonstrating better unlearning effectiveness compared to the others.

\begin{figure}[tbp]
	\centering
	\subfigure[Impact of momentum $\beta$]{
		\begin{minipage}{0.42\columnwidth} 
            \includegraphics[width=\textwidth]{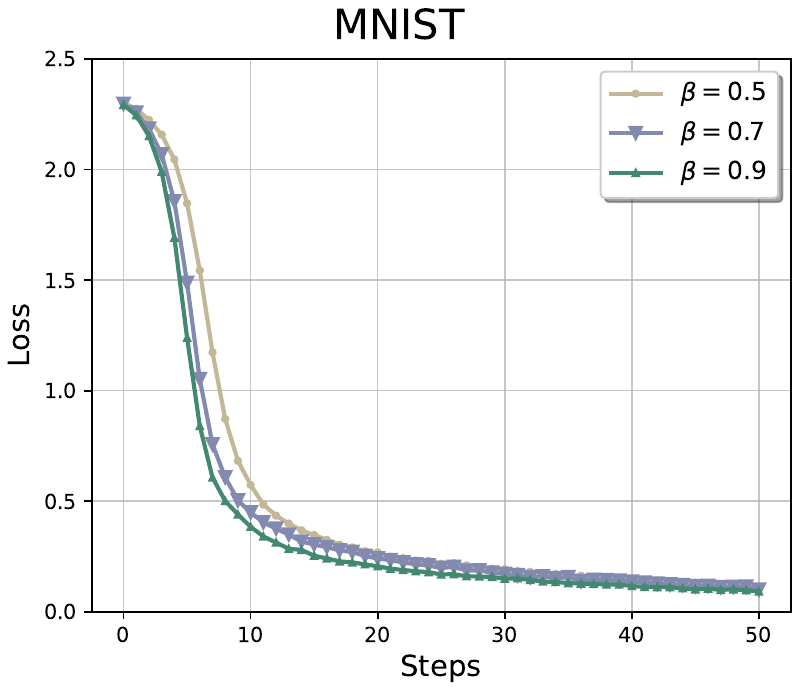} \\
            \label{fig:momentum}
		\end{minipage}
	}
	\subfigure[Impact of stopping criterion $\lambda$]{
		\begin{minipage}{0.44\columnwidth}
			\includegraphics[width=\textwidth]{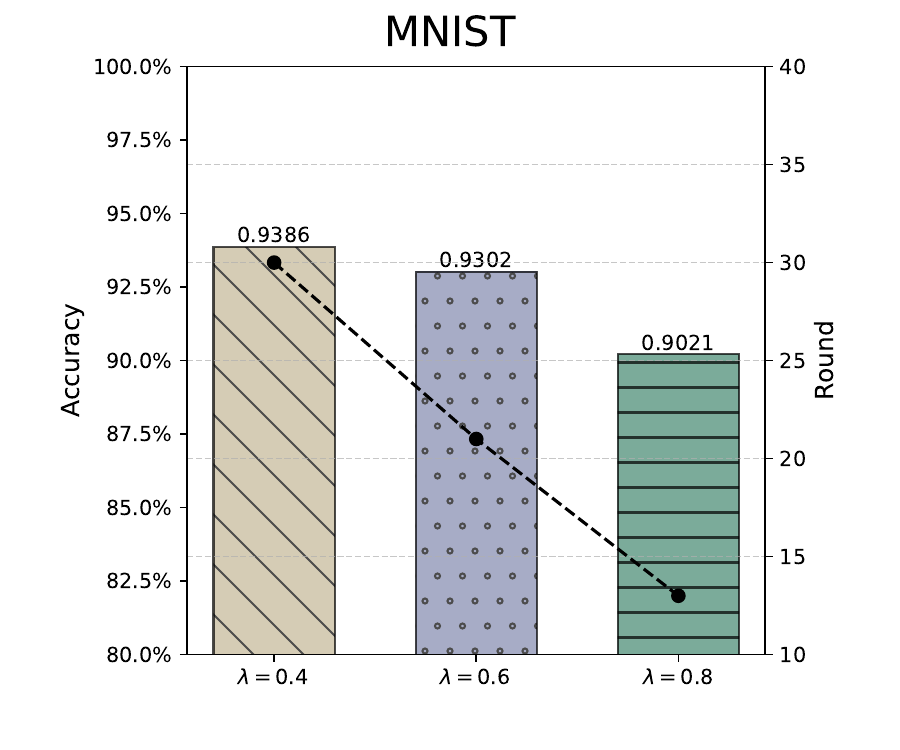}\\
			\label{fig:stop}
		\end{minipage}
	}
	\caption{Impact of parameters on MNIST}
    \label{fig:backdoor}
\end{figure}

\subsubsection{Discussion}
We discuss the impact of the hyperparameters $\beta$ and $\lambda$ on model performance. As seen in Fig. \ref{fig:momentum}, a higher momentum parameter $\beta$ accelerates convergence speed. Fig. \ref{fig:stop} illustrates that a larger stopping criterion parameter $\lambda$ results in stricter stopping conditions, fewer training rounds, and correspondingly lower accuracy. However, when $\lambda$ is set to 0.4, the stopping condition is more relaxed, allowing for more training rounds, specifically 30 rounds, and achieving an accuracy of 93.86\%. With $\lambda$ at 0.6, the model reaches an accuracy of 93.02\% in just 22 rounds. Selecting the appropriate $\lambda$ can save on communication and computational costs while ensuring satisfactory accuracy.

\section{Conclusion}
We introduce FedUHB, a novel exact unlearning method designed to achieve rapid retraining. By integrating Polyak heavy ball optimization techniques and a dynamic stopping mechanism, FedUHB achieves faster convergence and higher accuracy than conventional methods, enabling unlearning tasks to be completed in significantly less time. 
Extensive experiments show that FedUHB enhances unlearning effectiveness and speed, as evaluated through standard metrics and two types of attacks, highlighting its robustness compared to existing methods.
These findings demonstrate that FedUHB is a highly effective and efficient method for exact unlearning.

\section*{Acknowledgement}
This research / project is supported by the National Research Foundation, Singapore and Infocomm Media Development Authority under its Trust Tech Funding Initiative and the Singapore Ministry of Education Academic Research Fund (RG91/22 and NTU startup). Any opinions, findings and conclusions or recommendations expressed in this material are those of the author(s) and do not reflect the views of National Research Foundation, Singapore and Infocomm Media Development Authority.


\clearpage

\clearpage
\bibliographystyle{IEEEtran}
\bibliography{mybibliography}

\end{document}